\documentclass[letterpaper, 10 pt, conference]{ieeeconf}

\IEEEoverridecommandlockouts

\usepackage{url}
\usepackage{cite}
\usepackage{amsmath}
\usepackage{amssymb}

\usepackage{amsthm}
\usepackage{graphicx}
\usepackage{multirow}

\newtheorem{lemma}{Lemma}
\newtheorem*{theorem}{Theorem}

\title{\LARGE \bf 
    GFM-Planner: Perception-Aware Trajectory Planning \\ with Geometric Feature Metric
}

\author{
    Yue Lin, Xiaoxuan Zhang, Yang Liu*, Dong Wang, Huchuan Lu
    \thanks{All authors are with Dalian University of Technology, Dalian 116024, China. *Corresponding author: Yang Liu, \tt\small{ly@dlut.edu.cn}.}
    \thanks{This work is supported by the National Natural Science Foundation of China under grant Nos. 62293542, U23A20384, 62476044, and 62388101.}
}

\begin{document}

\maketitle
\thispagestyle{empty}
\pagestyle{empty}

\begin{abstract}

Like humans who rely on landmarks for orientation, autonomous robots depend on feature-rich environments for accurate localization. In this paper, we propose the GFM-Planner, a perception-aware trajectory planning framework based on the geometric feature metric, which enhances LiDAR localization accuracy by guiding the robot to avoid degraded areas. First, we derive the Geometric Feature Metric (GFM) from the fundamental LiDAR localization problem. Next, we design a 2D grid-based Metric Encoding Map (MEM) to efficiently store GFM values across the environment. A constant-time decoding algorithm is further proposed to retrieve GFM values for arbitrary poses from the MEM. Finally, we develop a perception-aware trajectory planning algorithm that improves LiDAR localization capabilities by guiding the robot in selecting trajectories through feature-rich areas. Both simulation and real-world experiments demonstrate that our approach enables the robot to actively select trajectories that significantly enhance LiDAR localization accuracy.

\end{abstract}

\section{Introduction}

Accurate autonomous localization is essential for robots operating in environments without GPS, especially when performing precise tasks. Autonomous localization methods in robotics are generally divided into visual localization~\cite{qin2018vins,geneva2020openvins,campos2021orb} and LiDAR-based localization~\cite{xu2021fast,xu2022fast,bai2022faster}. While these algorithms can provide reliable pose estimates in most scenarios, localization accuracy may suffer in environments with sparse features. Perception-aware trajectory planning aims to enable robots to actively select trajectories that enhance localization accuracy, minimizing errors induced by feature sparsity.

Although perception-aware planning frameworks based on vision~\cite{zhang2019beyond,bartolomei2020perception,chen2024apace} have reached considerable maturity, they are limited by the inherent challenges of visual localization, such as variations in lighting~\cite{falanga2018pampc} and texture loss~\cite{costante2016perception}. In contrast, LiDAR can perceive geometric features even in texture-sparse environments and is less affected by lighting changes, leading to more robust localization performance. However, research on LiDAR-based perception-aware planning remains limited~\cite{chai2024lf}.

The goal of LiDAR-based perception-aware planning is to guide robots away from areas with insufficient geometric features, which we refer to as \textit{degraded areas}. Fig.~\ref{fig:head}(a) illustrates a typical scenario involving a degraded area. In this example, when navigating from point S to point G, the blue trajectory generated by the perception-aware planning algorithm actively avoids the degraded area, instead navigating through regions rich in geometric features that support better LiDAR localization. Therefore, to achieve perception-aware trajectory planning, the key challenge is to design a metric that quantifies the localization potential of each pose.

\begin{figure}
    \vspace{2 mm}
    \centering
    \includegraphics[width=\linewidth]{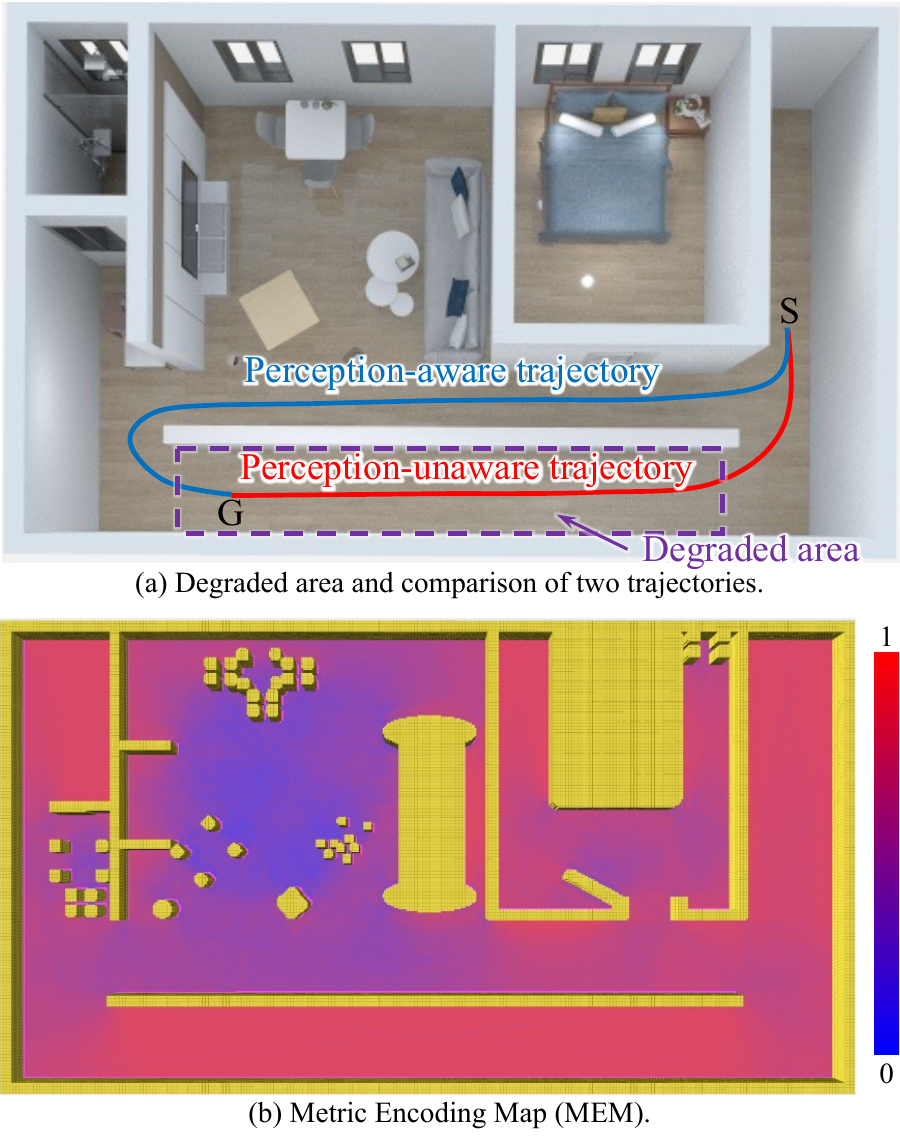}
    \caption{Visualizations of degraded areas and the MEM. (a) The blue trajectory represents the trajectory generated by our perception-aware trajectory planning algorithm, which actively avoids the degraded area and traverses regions with rich geometric features. In contrast, the trajectory produced by the perception-unaware planning algorithm directly enters the degraded area. (b) The degraded areas are assigned high metric values, while regions with rich geometric features are assigned low metric values. We enhance the trajectory's suitability for localization by minimizing its metric value.}
    \label{fig:head}
\end{figure}

For LiDAR localization metrics, disturbance analysis-based approaches have been explored~\cite{chai2024lf,zhang2016degeneracy}. While these methods yield satisfactory results, they require pre-storing metrics for all discrete poses and utilize only a small subset of them during trajectory optimization. This results in a cumbersome and inefficient algorithm. To address this issue, we derive the Geometric Feature Metric (\textbf{GFM}) from the fundamental LiDAR localization problem to provide a more efficient and effective method for quantifying the localization potential of each position.

In practical applications, real-time simulation of LiDAR scans can be computationally expensive. Inspired by the LF-3PM~\cite{chai2024lf}, we discretize the robot’s position into a 2D grid. GFM values are computed at the corner points of each grid cell, and arbitrary positions within the grid can obtain their corresponding GFM values via linear interpolation. Unlike the LF-3PM, which uses a 3D grid map to store metrics for both position and orientation, we employ a 2D grid map where each cell stores a metric encoding that can be decoded to obtain GFM values at any orientation in constant time. This 2D grid map is referred to as the Metric Encoding Map (\textbf{MEM}). Fig.~\ref{fig:head}(b) provides an intuitive visualization of the MEM. Degraded areas, such as long, featureless corridors, are considered disadvantageous for localization, whereas regions rich in geometric features provide significant advantages for localization.

To meet the practical requirements of the robot, including limited LiDAR perception range, obstacle avoidance, and energy efficiency, we propose a heuristic perception-aware path searching algorithm based on the MEM for generating perception-aware reference paths. Subsequently, we optimize the path to generate a trajectory that enhances LiDAR localization by considering multiple factors, including the GFM value, energy cost, safety constraints, and motion feasibility. Fig. 2 illustrates the overall framework of our system, named the \textbf{GFM-Planner}. This meticulously designed perception-aware trajectory planning framework is highly efficient and capable of performing real-time replanning.

We conduct extensive experiments in both simulated and real-world environments to demonstrate the practicality and effectiveness of the proposed perception-aware trajectory planning framework. In the real-world experiment, we compare our framework with a state-of-the-art perception-aware planning method~\cite{chai2024lf}. The experimental results demonstrate that, due to the reliable GFM and an efficient trajectory planning algorithm, the robot can plan trajectories through feature-rich areas in real time, thereby significantly reducing the LiDAR localization error, as shown in Fig.~\ref{fig:realworld}. In addition, the ablation experiments presented in Fig.~\ref{fig:ablation} and Tab.~\ref{tab} highlight the importance of each module within our perception-aware trajectory planning framework.

In summary, the major contributions of this paper are:

\begin{itemize}
    \item A theoretically grounded Geometric Feature Metric derived from the fundamental LiDAR localization problem, quantifying the localizability of poses.
    \item An efficient Metric Encoding Map that stores the GFM values in a 2D grid, enabling real-time metric decoding with constant complexity.
    \item A holistic perception-aware planning framework integrating MEM-based path search and trajectory optimization, validated in real-world scenarios.
    \item Extensive real-world experiments and ablation studies are conducted to validate the practicality and effectiveness of the proposed method. To facilitate further research, we make our source code publicly available\footnote{\url{https://github.com/Yue-0/GFM-Planner}.}.
\end{itemize}

\section{Related Work}

In recent years, vision-based perception-aware planning methods have achieved substantial advancements. Early approach~\cite{costante2016perception} explored paths that reduce localization uncertainty by using the covariance as a heuristic function within rapidly-exploring random trees. The perception-aware receding horizon method~\cite{zhang2018perception} selects the final trajectory with the lowest covariance from multiple candidate trajectories that satisfy obstacle avoidance and kinematic constraints. However, the covariance of a state depends on historical observations gathered from other states, and simulating these observations is computationally expensive, resulting in a prolonged trajectory generation process. The Fisher information field~\cite{zhang2019beyond} models visual feature metrics as differentiable functions of the robot’s pose, enabling the generation of trajectories that enhance visual localization. However, this method primarily focuses on analyzing texture richness, making it inapplicable to LiDAR-based perception-aware planning, which emphasizes the geometric structure of the environment rather than visual texture. Similarly, the APACE~\cite{chen2024apace} optimizes trajectories by maximizing feature matching between frames, offering excellent performance for vision-based perception-aware planning. Unfortunately, in the context of LiDAR, the high density of point clouds makes it impractical to simulate scans during motion planning.

For LiDAR-based perception-aware planning, the receding horizon method~\cite{takemura2022perception} introduced an empirical metric for evaluating LiDAR observations. However, this metric lacks a solid theoretical foundation, limiting its adaptability to diverse scenarios and sensor configurations. The LF-3PM~\cite{chai2024lf} established a more robust metric through perturbation analysis. Nonetheless, this method requires a 3D grid map to pre-store metric values for all poses, and only utilizes a small subset of them during trajectory optimization. This discrepancy leads to significant resource inefficiency. In contrast, our approach employs the 2D MEM to store metric encoding, drastically reducing memory overhead. The MEM is fully utilized throughout both the path search and trajectory optimization stages, enabling the efficient generation of trajectories that enhance LiDAR localization.

\section{Geometric Feature Metric}

\subsection{Start with the Fundamental LiDAR Localization Problem}

In the $SE(2)$ space, the robot's pose is represented as $\mathbf{p} = [x, y, \theta]^\top$. The coordinates of the $i$-th scanning point of the LiDAR at the pose $\mathbf{p}$ are denoted as $\mathbf{L}_i(\mathbf{p}) \in \mathbb{R}^2$. The distance from point $\mathbf{q} \in \mathbb{R}^2$ to its nearest obstacle boundary is defined as $E(\mathbf{q}) \in \mathbb{R}$. Here, we assume that both $E$ and $\mathbf{L}$ are sufficiently smooth functions and that $\nabla E \neq \mathbf{0}$.

The fundamental LiDAR localization problem can be formulated as an optimization task, where the objective is to minimize the alignment error between the LiDAR scan and the global environment. This can be expressed as

\begin{equation}
    \underset{\mathbf{p}}{\min} \quad f \left( \mathbf{p} \right) = \dfrac{1}{2} \sum_{i=1}^{N} E \left( \mathbf{L}_i \left( \mathbf{p} \right) \right)^2,
    \label{eq:problem}
\end{equation}
where $N$ represents the number of scanning points. Let $\mathbf{q}_i = \mathbf{L}_i(\mathbf{p})$, and denote the Jacobian matrix of $\mathbf{q}_i$ with respect to $\mathbf{p}$ as $\mathbf{J}_i(\mathbf{p}) \in \mathbb{R}^{2 \times 3}$. The gradient of the objective function $f$ with respect to $\mathbf{p}$ is then given by
\begin{equation}
    \nabla f \left( \mathbf{p} \right) = \sum_{i=1}^{N} E \left( \mathbf{q}_i \right) \mathbf{J}_i^\top \left( \mathbf{p} \right) \nabla E \left( \mathbf{q}_i \right).
\end{equation}

To further analyze the problem, we derive the Hessian matrix of $f$, which is expressed as
\begin{equation}
    \nabla^2 f \left( \mathbf{p} \right) = \mathbf{H}_1 + \mathbf{H}_2 + \mathbf{H}_3,
\end{equation}
where the individual components are defined as
\begin{subequations}
    \begin{equation}
        \mathbf{H}_1 = \sum_{i=1}^{N} \left( \mathbf{J}_i^\top \left( \mathbf{p} \right) \nabla E \left( \mathbf{q}_i \right) \right) \left( \mathbf{J}_i^\top \left( \mathbf{p} \right) \nabla E \left( \mathbf{q}_i \right) \right)^\top,
    \end{equation}
    \vspace{-2.5 mm}
    \begin{equation}
        \mathbf{H}_2 = \sum_{i=1}^{N} E \left( \mathbf{q}_i \right) \mathbf{J}_i^\top \left( \mathbf{p} \right) \nabla^2 E \left( \mathbf{q}_i \right) \mathbf{J}_i \left( \mathbf{p} \right),
    \end{equation}
    \vspace{-2.5 mm}
    \begin{equation}
        \mathbf{H}_3 = \sum_{i=1}^{N} E \left( \mathbf{q}_i \right) \nabla E \left( \mathbf{q}_i \right)^\top \nabla \mathbf{J}_i \left( \mathbf{p} \right).
    \end{equation}
\end{subequations}

\subsection{Positive Definiteness Analysis}

To ensure that the optimization problem (\ref{eq:problem}) converges to a unique local optimal solution, the function $f$ must be strictly convex in the neighborhood of $\mathbf{p}$. This requirement is equivalent to the Hessian matrix $\nabla^2 f$ being positive definite near $\mathbf{p}$. In this section, we derive a sufficient condition for $\nabla^2 f$ to be positive definite.

\begin{theorem}
    If the Jacobian matrix $\mathbf{J}_i \left( \mathbf{p} \right)$ is full rank for all $i \in \left\lbrace 1, 2, \cdots, N \right\rbrace$, then $\mathbf{H}_1$ is positive definite.
\end{theorem}

\begin{proof}
    We first demonstrate that $\mathbf{H}_1$ is positive semi-definite. Given that $\mathbf{J}_i \left( \mathbf{p} \right)$ is full rank, we have
    \begin{equation}
        \mathbf{J}_i^\top \left( \mathbf{p} \right) \nabla E \left( \mathbf{q}_i \right) \ne \mathbf{0}.
    \end{equation}
    We can treat this vector as $\mathbf{a}_i$ in Lemma~\ref{lemma:semi} of Appendix.~\ref{app:lemma}, which establishes that $\mathbf{H}_1$ is positive semi-definite. Next, we prove by contradiction that $\mathbf{H}_1$ is positive definite. Assume $\mathbf{H}_1$ is not positive definite. According to Lemma~\ref{lemma:positive}, there exists a vector $\boldsymbol{\xi} = \left[ \xi_1, \xi_2, \xi_3 \right]^\top \ne \mathbf{0}$ such that
    \begin{equation}
        \boldsymbol{\xi}^\top \mathbf{J}_i^\top \left( \mathbf{p} \right) \nabla E \left( \mathbf{q}_i \right) = 0,\ \forall i \in \left\lbrace 1, 2, \cdots, N \right\rbrace.
    \end{equation}
    Expanding $\mathbf{J}_i^\top \left( \mathbf{p} \right) \nabla E \left( \mathbf{q}_i \right)$, we have
    \begin{equation}
        \mathbf{J}_i^\top \left( \mathbf{p} \right) \nabla E \left( \mathbf{q}_i \right) = \left[ \dfrac{\partial E \left( \mathbf{q}_i \right)}{\partial x}, \dfrac{\partial E \left( \mathbf{q}_i \right)}{\partial y}, \dfrac{\partial E \left( \mathbf{q}_i \right)}{\partial \theta} \right]^\top.
    \end{equation}
    Thus,
    \begin{equation}
        \xi_1 \dfrac{\partial E \left( \mathbf{q}_i \right)}{\partial x} + \xi_2 \dfrac{\partial E \left( \mathbf{q}_i \right)}{\partial y} + \xi_3 \dfrac{\partial E \left( \mathbf{q}_i \right)}{\partial \theta} = 0.
        \label{eq:partial}
    \end{equation}
    This homogeneous linear partial differential equation implies that $E \left( \mathbf{q}_i \right)$ has contour lines~\cite{evans2022partial}. Let $l$ denote the contour line defined by $E \left( \mathbf{q} \right) = 0$. Suppose the $n$-th LiDAR ray $l_n$ originating from $\mathbf{p}_0^\top = \left( x_0, y_0, \theta_0 \right)$ intersects $l$ at $\mathbf{L}_n \left( \mathbf{p}_0 \right) = (a, b)^\top$. The ray equation is $y = k(x - a) + b$. Performing a first-order Taylor expansion of $l$ at $\mathbf{L}_n \left( \mathbf{p}_0 \right)$, we express $l$ as
    \begin{equation}
        A(x-a) + B(y-b) = 0.
        \label{eq:taylor}
    \end{equation}
    The Jacobian matrix of $\mathbf{L}_n \left( \mathbf{p}_0 \right)$ is given by (Details can be found in Appendix.~\ref{app:details})
    \begin{equation}
        \mathbf{J}_n \left( \mathbf{p}_0 \right) = \dfrac{1}{A+Bk} \begin{bmatrix}
            Bk & -B & (1+k^2)(x_0-a)B \\
            -Ak & A & (1+k^2)(a-x_0)A \\
        \end{bmatrix}.
        \label{eq:result}
    \end{equation}
    It is evident that the first and second rows of $\mathbf{J}_n \left( \mathbf{p}_0 \right)$ are proportional, contradicting the assumption that $\mathbf{J}_i \left( \mathbf{p} \right)$ is full rank. Thus, $\mathbf{H}_1$ is positive definite.
\end{proof}

In practice, the optimization problem (\ref{eq:problem}) is often solved iteratively using the Gauss-Newton method, where $E(\mathbf{q}_i) \approx 0$. Consequently, $\mathbf{H}_2 + \mathbf{H}_3$ can be treated as a small perturbation of the positive definite matrix $\mathbf{H}_1$. Since the set of positive definite matrices is open, this perturbation results in another positive definite matrix~\cite{mathias1997spectral}. Therefore, if $\mathbf{J}_i \left( \mathbf{p} \right)$ remains full rank in the neighborhood of $\mathbf{p}$, the function $f$ is strictly convex in this region.

\subsection{Metric Design}

In this section, we propose the GFM based on the positive definiteness condition established in the previous section. Since the Jacobian matrix $\mathbf{J}_i \in \mathbb{R}^{2 \times 3}$, its rank can only be either one or two. We argue that an insufficient rank of $\mathbf{J}_i$ hinders LiDAR localization, as it results in $f$ not being strictly convex. This, in turn, leads to a non-unique local optimal solution to the optimization problem (\ref{eq:problem}). Consequently, we define the GFM value for each position as
\begin{equation}
    M \left( \mathbf{p} \right) = 2N - \sum_{i=1}^{N} \text{rank} \left( \mathbf{J}_i \left( \mathbf{p} \right) \right).
    \label{eq:rank}
\end{equation}
Here, ${M}$ encapsulates the richness of the geometric features and reflects how well the position facilitates LiDAR-based localization. A smaller value of $M$ indicates that the position provides better support for LiDAR localization, enhancing the accuracy of the localization process.

\section{Perception-Aware Trajectory Planning}

\subsection{Framework Overview}

In this section, we introduce our perception-aware trajectory planning framework, as illustrated in Fig. 2. The framework consists of three key modules: a metric encoding module, a heuristic pre-search module, and a perception-aware planning module. To clarify the terminology, we differentiate between \textit{path} and \textit{trajectory} as described in~\cite{gasparetto2015path}. A path refers to a geometric route, while a trajectory is a time-parameterized path.

In the metric encoding module, we propose a pre-built MEM that is stored prior to the system’s operation. In the heuristic pre-search module, once the user specifies the target navigation pose, we employ a single-source shortest path search algorithm. The output of this algorithm is subsequently used to calculate the heuristic function during each replanning phase. In the perception-aware planning module, we first perform heuristic perception-aware path searching to obtain a reference path. We then parameterize the path and execute perception-aware trajectory optimization.

\begin{figure*}
    \vspace{2 mm}
    \centering
    \includegraphics[width=\linewidth]{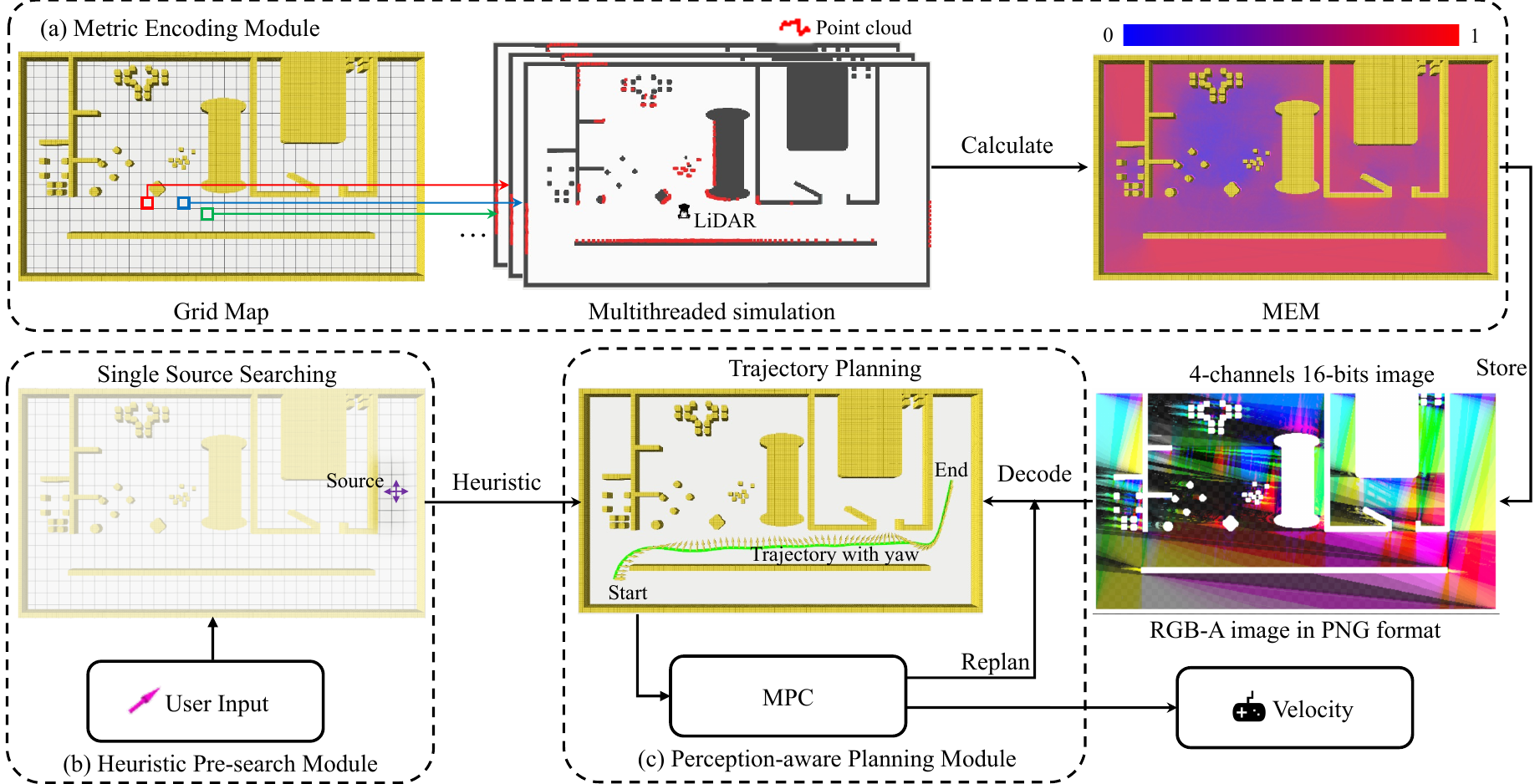}
    \caption{Overview of our perception-aware trajectory planning framework. (a) The metric encoding module encodes the GFM values in the MEM and stores it as a PNG format image. (b) After the user specifies the navigation target pose, the heuristic pre-search module pre-acquires the heuristic function values for each grid. (c) The perception-aware trajectory planning module decodes the GFM in real time to plan a trajectory that enhances LiDAR localization.}
    \label{fig:framework}
\end{figure*}

\subsection{Metric Encoding Map}

\subsubsection{Metric Encoding}

As mentioned previously, simulating LiDAR scans during the planning process is computationally expensive. To address this, we pre-simulate the LiDAR scan results on the grid map, calculate the GFM values for each position, and encode these values into the MEM. Specifically, the MEM has the same size and resolution as the grid map. For each obstacle-free grid cell $(x, y)$ in the map, we discretize the robot’s yaw angle into $L$ distinct angles and simulate the LiDAR scan at each angle. We then compute the rank of the Jacobian matrix $\mathbf{J}_i(x, y)$ corresponding to each LiDAR scan point. Since the rank can only be one or two, we use an $L$-bit integer $q(x, y)$ to encode the rank of the $L$ Jacobian matrices:
\begin{equation}
    q(x, y) = \sum_{i=1}^{L} 2^i - 2^{i-1}\mathrm{rank} \left( \mathbf{J}_i(x, y) \right).
\end{equation}

This encoding ensures that when $\mathbf{J}_i(x, y)$ is of full rank, the $i$-th binary bit of $q(x, y)$ is set to zero. Otherwise, the $i$-th bit is set to one. For grid cells occupied by obstacles, we assign $q(x, y) = 2^L - 1$, indicating that all binary bits are set to one. 

The calculation for each grid is independent, allowing for parallel execution. Once the $q$ values for all obstacle-free grids are obtained, the MEM construction is complete. 

In practice, we select $L = 64$, which allows $q$ to correspond to an unsigned long integer. The MEM can then be stored as a 16-bit, 4-channel image, which is easily processed using mainstream image processing libraries. This format supports efficient trajectory planning tasks.

\subsubsection{Metric Decoding}

In practice, the LiDAR scanning range is limited. Suppose that, for a given yaw angle, the scanning range of the LiDAR covers the $i$-th to  $j$-th discrete angles\footnote{Due to the periodic nature of the angles, we cannot always determine the magnitude relationship between $i$ and $j$. When $i > j$ , it indicates that the LiDAR range encompasses the discrete angles from the 1st to the  $i$-th and from the $j$-th to the $L$-th angles.}. The $q$ value at this time is expressed as
\begin{equation}
    q_{ij}(x, y) = \begin{cases}
        \left( 2^j - 2^{i-1} \right) \circ q(x, y), & i \le j, \\
        \lnot \left( 2^i - 2^{j-1} \right) \circ q(x, y), & i > j,
    \end{cases}
    \label{eq:metric}
\end{equation}
where $\circ$ represents the bitwise AND operation, and $\lnot$ represents the bitwise inversion operation. Both operations are applied to $L$-bit unsigned integers.

During the trajectory planning process, it is essential to decode the metric value $M_{ij}$ from $q_{ij}$ in real time. By combining the equations (\ref{eq:rank})$-$(\ref{eq:metric}), we derive the relationship between $M_{ij}$ and $q_{ij}$ as
\begin{equation}
    M_{ij}(x, y) = \mathcal{W}(q_{ij}(x, y)),
\end{equation}
where $\mathcal{W}(\chi)$ denotes the Hamming weight of the integer $\chi$, which refers to the number of ones in its binary representation. This value can be computed in constant time.

For continuous robot poses and LiDAR scan ranges, we use quartic linear interpolation to determine the measurement value for any LiDAR scan range when the robot is at a given pose. This encoding and decoding method is designed to ensure that the MEM allows us to obtain localization metric values with consistent time complexity, thereby optimizing the efficiency of trajectory planning.

\subsection{Perception-Aware Path Planning}

The perception-aware path planning method is based on a hybrid-state A* algorithm~\cite{dolgov2010path}, which expands nodes generated by syncretizing the control input to search for safe and dynamically feasible paths on a grid map. To integrate perception-awareness into the search process, we redesign both the cost function and the heuristic function to fully leverage the GFM.

In the hybrid-state A* algorithm, the cost function for a node is represented as $f_c = g_c + h_c$, where $g_c$ denotes the actual cost from the initial pose $\mathbf{p}_0$ to the current pose $\mathbf{p}_c$, and $h_c$ represents the heuristic cost, which aims to expedite the search process.

Suppose it takes $n_c$ poses to transition from the initial pose $\mathbf{p}_0$ to the current pose $\mathbf{p}_c$, where the $i$-th pose is denoted as $\mathbf{p}_{ci}$. To enhance the accuracy of LiDAR localization, we define the actual cost as
\begin{equation}
    g_c = \sum_{i=1}^{n_c} \sigma_\varepsilon \left( M \left( \mathbf{p}_{ci} \right) \right),
\end{equation}
where $\varepsilon$ is a positive hyper-parameter, and $\sigma_\varepsilon: [0, L] \mapsto (0, 1)$ is the Sigmoid smoothing function:
\begin{equation}
    \sigma_\varepsilon(m) = \left( 1 + \exp \left( \dfrac{\varepsilon L - 2 \varepsilon m}{L} \right) \right)^{-1}.
\end{equation}

For the heuristic cost, we designate the grid containing the goal pose $\mathbf{p}_g$ as the source grid $s$, and apply a single-source shortest path search algorithm for directed weighted graphs~\cite{dijkstra2022note} on the MEM. The weight between grid $a$ and its adjacent grid $b$ is set to the GFM value of grid $b$. This allows us to compute the path from each grid on the MEM to the source grid, ensuring that the sum of the sigmoid values of the GFM values of all grids along the path is minimized. For each grid, this value is stored as the heuristic function value for that grid, expressed as
\begin{equation}
    h_c = \underset{P} \min \sum_{(x, y) \in P} \sigma_\varepsilon (M_{1L}(x, y)),
\end{equation}
where $P$ represents the path from grid $s$ to the current grid. It is important to note that since the pre-built MEM remains unchanged during the planning process, the single-source shortest path search algorithm only needs to be executed once when the user specifies the goal pose. In each subsequent replanning process, we can directly query the heuristic function value of any grid.

\subsection{Perception-Aware Trajectory Optimization}

\subsubsection{Trajectory Representation}

After obtaining the reference path through perception-aware path planning, we parameterize it into a trajectory for optimization. First, we extract key poses $\lbrace \mathbf{q}_0, \mathbf{q}_1, \cdots, \mathbf{q}_K \rbrace$ along the path using the method described in~\cite{lin2024safety}, where $\mathbf{q}_0$ and $\mathbf{q}_K$ represent the initial pose and the goal pose, respectively. We then parameterize the path into a MINCO trajectory~\cite{wang2022geometrically} using these key poses:
\begin{equation}
    \mathfrak{T} = \lbrace \mathbf{p}(t): \mathbb{R}_+ \mapsto \mathbb{R}^3 \ | \ \mathbf{C} = \mathcal{M} \left( \mathbf{Q}, \mathbf{t} \right) \rbrace,
\end{equation}
where $\mathbf{p}(t)$ is a polynomial of order $2s - 1$ with $K$ segments, and $\mathbf{Q} = \left[ \mathbf{q}_1, \mathbf{q}_2, \cdots, \mathbf{q}_{K-1} \right] \in \mathbb{R}^{3 \times (K-1)}$ represents the intermediate way points. The vector $\mathbf{t} = \left[ t_1, t_2, \cdots, t_K \right]^\top \in \mathbb{R}_+^K$ denotes the duration of each trajectory segment, while $\mathbf{C} = \left[ \mathbf{C}_1, \mathbf{C}_2, \cdots, \mathbf{C}_K \right] \in \mathbb{R}^{K \times 2s \times 3}$ contains the coefficients for the trajectory segments. The function $\mathcal{M}$ maps $\mathbf{Q}$ and $\mathbf{t}$ to $\mathbf{C}$ with linear complexity, allowing any second-order continuous cost function $J \left( \mathbf{C}, \mathbf{t} \right)$ to have a gradient applicable to the MINCO trajectory~\cite{wang2022geometrically}. The $i$-th segment of $\mathbf{p}$ is denoted as
\begin{equation}
    \mathbf{p}_i(t) = \left[ x_i(t), y_i(t), \theta_i(t) \right]^\top = \mathbf{C}_i^\top \boldsymbol{\beta}(t),
\end{equation}
where $\boldsymbol{\beta}(t) = \left[1, t, t^2, \cdots, t^{2s-1} \right]^\top$ is the natural basis.

\subsubsection{Problem Formulation}

The perception-aware trajectory optimization problem can be formulated as
\begin{subequations}
    \begin{align}
        \underset{\mathbf{Q}, \mathbf{t}}{\min} \quad & \lambda_l J_l + \lambda_e J_e, \label{eq:opt}\\
        \mathrm{s.t.} \quad & t_i > 0, \label{eq:time}\\
        \quad & E(\mathbf{\Pi}_2 \mathbf{p}_i(t)) \ge d, \label{eq:safe}\\
        \quad & \left \vert \mathbf{e}_3^\top \mathbf{p}_i^{(j)}(t) \right \vert \le \omega_j, \forall j \in \lbrace 1, 2 \rbrace,\label{eq:vel1}\\
        \quad & \left \Vert \mathbf{\Pi}_2 \mathbf{p}_i^{(j)}(t) \right \Vert \le v_j, \forall j \in \lbrace 1, 2 \rbrace,\label{eq:vel2}
    \end{align}
\end{subequations}
where $J_l$ is the localization cost term, $J_e$ is the energy cost term, and $\lambda_l$ and $\lambda_e$ are weights for these terms. Constraint (\ref{eq:time}) ensures that each element of the vector $\mathbf{t}$ lies within the time manifold $\mathbb{R}_+$, providing physical meaning. Constraint (\ref{eq:safe}) ensures that the robot avoids collisions with obstacles, where $\mathbf{\Pi}_2 = \left( \mathbf{e}_1, \mathbf{e}_2 \right)^\top$, $\mathbf{e}_i$ denotes the $i$-th column of the third-order identity matrix, and $d$ is a hyper-parameter representing the safety distance. Constraints (\ref{eq:vel1})$-$(\ref{eq:vel2}) ensure kinematic feasibility, with $v_1$, $v_2$, $\omega_1$, and $\omega_2$ representing the velocity, acceleration, angular velocity, and angular acceleration limits of the robot, respectively.

\subsubsection{Localization Cost}

The localization cost seeks to minimize the GFM value of all poses along the trajectory, which can be written as
\begin{equation}
    J_l = \sum_{i=1}^{K} \int_{0}^{t_i} \sigma_\varepsilon \left( M \left( \mathbf{p}_i(t) \right) \right) \mathrm{d}t.
\end{equation}
We calculate this cost through numerical integration:
\begin{subequations}
    \begin{equation}
        J_l  \approx  \sum_{i=1}^{K} \dfrac{t_i}{k_i} \sum_{j=0}^{k_i} \eta_j \sigma_\varepsilon \left( \mu_{ij} \right),
    \end{equation}
    \begin{equation}
        \mu_{ij} = M \left( \mathbf{p}_i \left( \dfrac{jt_i}{k_i} \right) \right),
    \end{equation}
\end{subequations}
where $\left( \eta_0, \eta_1, \eta_2, \cdots, \eta_{k_i - 1}, \eta_{k_i} \right) = (0.5, 1, 1, \cdots, 1, 0.5)$, and $k_i$ represents the number of samples for the $i$-th numerical integration. We present the derivative of $\sigma_\varepsilon$ as
\begin{equation}
    \sigma_\varepsilon'(\mu) = \dfrac{2\varepsilon L^{-1}}{2 + \exp \left( \dfrac{2\varepsilon \mu}{L}-\varepsilon \right) + \exp \left( \varepsilon - \dfrac{2\varepsilon \mu}{L} \right)}.
\end{equation}
Then, the gradient of $J_l$ with respect to $\mathbf{C}_i$ and $t_i$ is
\begin{subequations}
    \begin{equation}
        \dfrac{\partial J_l}{\partial \mathbf{C}_i} = \dfrac{t_i}{k_i} \sum_{j=0}^{k_i} \eta_j \sigma_\varepsilon' \left( \mu_{ij} \right) \boldsymbol{\beta} \left( \dfrac{jt_i}{k_i} \right) \boldsymbol{\delta}_{ij}^\top,
    \end{equation}
    \begin{equation}
        \dfrac{\partial J_l}{\partial t_i} = \sum_{j=0}^{k_i} \dfrac{\eta_j \sigma_\varepsilon \left( \mu_{ij} \right)}{k_i} + \dfrac{jt_i\eta_j \sigma_\varepsilon' \left( \mu_{ij} \right)}{k_i^2} \boldsymbol{\delta}_{ij}^\top \mathbf{p}_i' \left( \dfrac{jt_i}{k_i} \right),
    \end{equation}
\end{subequations}
where $\boldsymbol{\delta}_{ij} = \nabla \mu_{ij}$, which can be obtained through quartic linear interpolation. Subsequently, the gradient is efficiently propagated to $\mathbf{Q}$ and $\mathbf{t}$ by the function $\mathcal{M}$.

\subsubsection{Energy Cost}

The energy cost is defined by the higher-order derivatives of the MINCO trajectory along with a time regularization term:
\begin{equation}
    J_e = \sum_{i=1}^{K} \int_{0}^{t_i} \Vert \mathbf{p}_i^{(s)}(t) \Vert_2^2 \mathrm{d}t + \rho t_i,
\end{equation}
where $\rho$ is the weight of the regularization term. The gradient of $J_e$ can be computed analytically~\cite{wang2022geometrically}.

\subsubsection{Constraints Elimination}

Constraint (\ref{eq:time}) is eliminated through a diffeomorphic transformation:
\begin{equation}
    \tau_i = \begin{cases}
        \sqrt{2t_i-1} - 1, & t_i > 1, \\
        1 - \sqrt{\dfrac{2}{t_i} - 1}, & 0 < t_i \le 1.
    \end{cases}
\end{equation}
For constraints (\ref{eq:safe})$-$(\ref{eq:vel2}), we construct penalty functions:
\begin{subequations}
    \begin{equation}
        G_d \left( \mathbf{p}, t \right) = \max \left\lbrace d - E \left( \mathbf{\Pi}_2 \mathbf{p}(t) \right), 0 \right\rbrace^s,
    \end{equation}
    \begin{equation}
        G_{v_j} \left( \mathbf{p}, t \right) = \max \left\lbrace \left \Vert \mathbf{\Pi}_2 \mathbf{p}^{(j)}(t) \right \Vert_2^2 - v_j^2, 0 \right\rbrace,
    \end{equation}
    \begin{equation}
        G_{\omega_j} \left( \mathbf{p}, t \right) = \max \left\lbrace \left( \mathbf{e}_3^\top \mathbf{p}^{(j)}(t) \right)^2 - \omega_j^2, 0 \right\rbrace.
    \end{equation}
\end{subequations}
These constraints are converted into a penalty term:
\begin{equation}
    J_G = \sum_{i=1}^{K} \dfrac{t_i}{k_i} \sum_{j=0}^{k_i} \eta_j \sum_{g \in G}  g \left( \mathbf{p}_i, \dfrac{jt_i}{k_i} \right),
\end{equation}
where $G = \lbrace G_d, G_{v_1}, G_{v_2}, G_{\omega_1}, G_{\omega_2} \rbrace$. Thus, the original problem is reformulated as
\begin{equation}
    \underset{\mathbf{Q}, \mathbf{\tau}}{\min} \quad \lambda_l J_l + \lambda_e J_e + \lambda_G J_G,
\end{equation}
where $\lambda_G$ is a sufficiently large positive number. This unconstrained optimization problem can be efficiently solved using the L-BFGS algorithm~\cite{liu1989limited} along with the Lewis-Overton line search~\cite{lewis2013nonsmooth}.

\section{Experiments}

\subsection{Real-World Experiment}

To demonstrate the practicality of our approach, we apply the proposed framework to an omnidirectional vehicle (RoboMaster AI-2020), as shown in Fig.~\ref{fig:realworld}(a). Furthermore, we compare the proposed framework with a state-of-the-art LiDAR-based perception-aware planning framework LF-3PM~\cite{chai2024lf}. We manually limit the LiDAR’s field of view to 90 degrees to simulate limited observations. 

In the experiment, we set the hyper-parameters to $s = 3$, $\rho = 20$, $\varepsilon = \lambda_e = \lambda_l = 1$ and $\lambda_G = 10^4$. The robot’s pose is estimated using the AMCL algorithm~\cite{thrun2002probabilistic}, with the ground truth position provided by the UWB localization module. We define the localization error as the Euclidean distance between these two positions. All computations are performed on the Jetson AGX Orin platform.

The robot navigates from point A to point B, and the trajectories generated by both our proposed framework and the LF-3PM are shown in Fig.~\ref{fig:realworld}(b). Since the LF-3PM considers the localization metric only during the trajectory optimization stage, the trajectories it generates are primarily guided by a perception-unaware reference path. Consequently, the trajectory optimizer cannot completely compensate for this limitation, leading to significant localization errors. In contrast, our method integrates the GFM in both the path search and trajectory optimization stages, resulting in trajectories that better support LiDAR localization. As illustrated in Fig.~\ref{fig:realworld}(c), the trajectory generated by our method effectively incorporates the geometric features of the environment, resulting in a trajectory with significantly lower localization error compared to the baseline.

\begin{figure}[h]
    \centering
    \includegraphics[width=\linewidth]{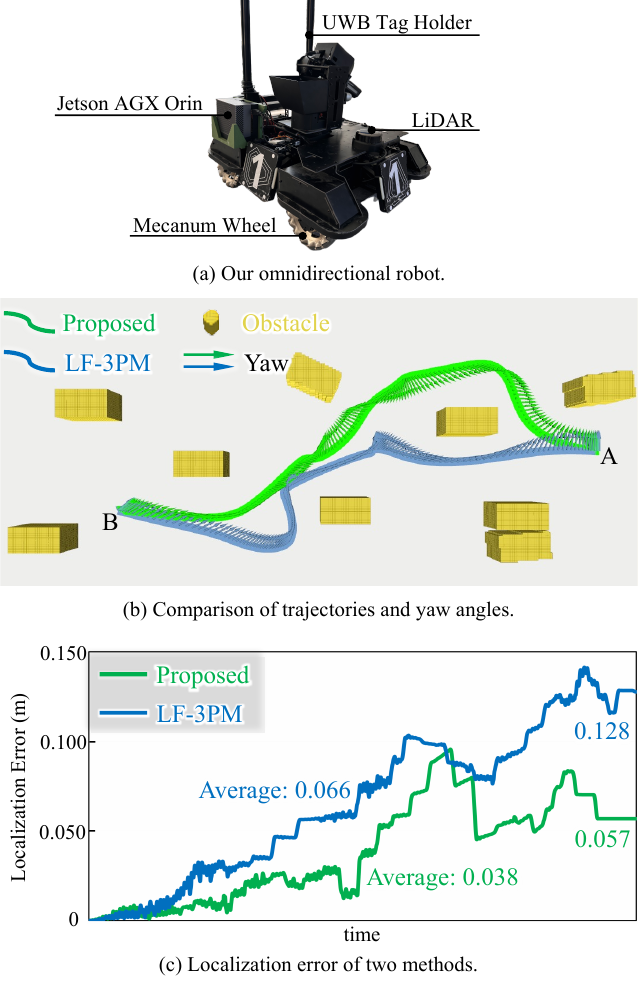}
    \caption{Description and results of real-world experiment.}
    \label{fig:realworld}
\end{figure}

\subsection{Ablation Studies}

In this section, we conduct ablation experiments to validate the necessity of both the proposed perception-aware path search and perception-aware trajectory optimization processes. First, we replace our perception-aware path search algorithm with a standard hybrid A* algorithm~\cite{dolgov2010path}. Additionally, we remove the localization cost $J_l$ from the trajectory optimization process, relying solely on the perception-aware reference path to guide the overall trajectory direction. All experiments are conducted in the simulated environment depicted in Fig.~\ref{fig:head}(a). The physical parameters of the robot and the field of view of the LiDAR are the same to those used in the real-world experiment. All computations are performed on a mobile device equipped with an Intel Core i7-1165G7 CPU and a GeForce RTX 2060 GPU, and the hyper-parameter settings are consistent with those of the real-world experiment.

We record the trajectory for each case in Fig.~\ref{fig:ablation}. Since the hybrid A* algorithm does not incorporate geometric feature awareness, the generated reference path can lead the robot into degraded areas, making it difficult for the local optimizer to correct this deficiency. Furthermore, eliminating the localization cost in trajectory optimization results in a trajectory that prioritizes energy efficiency while neglecting perception constraints. Although the trajectory does not enter a degraded area, the robot's yaw angle does not account for geometric features, which fails to significantly enhance localization accuracy when the LiDAR perception field of view is restricted. In contrast, the complete planner effectively balances perceptual and energy costs. The perception-aware path searching allows the robot to circumvent degraded areas, while the optimized yaw angle enhances the LiDAR to perceive rich geometric features, thereby significantly improving localization accuracy.

We repeat the experiment twenty times for each method. Tab.~\ref{tab} presents the mean localization error of three methods during navigation, as well as the average deviation between the robot's position upon reaching the endpoint and the goal position. Additionally, we record the computation time for global trajectory planning for each method. The results demonstrate that the complete planner significantly reduces the robot’s positioning error during navigation, allowing the robot to reach the target with greater accuracy. Furthermore, the computational time for trajectory planning is comparable to that of other methods. These findings highlight the critical importance of both perception-aware path searching and localization cost in achieving robust and localization-enhanced trajectory planning. Both components are essential, the absence of either results in performance degradation.

\begin{table}[h]
    \centering
    \caption{Results of Ablation Experiment}
    \begin{tabular}{l c c c}
        \hline
        Method                & Localization Error & Deviation        & Duration          \\
        \hline
        w/o path searching    & 0.057 m            & 0.176 m          & \textbf{0.197} s  \\
        w/o localization cost & 0.090 m            & 0.210 m          & 0.202 s           \\
        \hline
        complete planner      & \textbf{0.032} m   & \textbf{0.104} m & 0.255 s           \\
        \hline
    \end{tabular}
    \label{tab}
\end{table}

\begin{figure}[t]
    \centering
    \vspace{2 mm}
    \includegraphics[width=\linewidth]{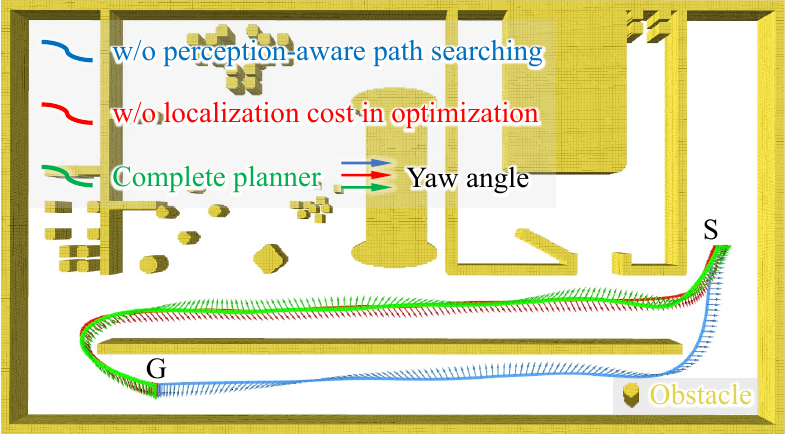}
    \caption{Comparison of trajectories and yaw angles in ablation experiment.}
    \label{fig:ablation}
\end{figure}

\section{Conclusion and Limitations}

In this paper, we propose a new perception-aware trajectory planning framework, referred to as the GFM-Planner, which effectively reduces the robot’s localization error during navigation by leveraging a reliable GFM and a well-designed trajectory planning process. Real-world experiments demonstrate that the proposed framework achieves superior localization accuracy compared to state-of-the-art methods. Furthermore, ablation experiments highlight the critical role of each individual module in the framework. We hope that this work will contribute to the practical advancement of perception-aware trajectory planners and inspire further research in the field of perception-aware planning. In the future, we plan to extend our method to the $SE(3)$ space and apply it to perception-aware trajectory planning for aerial and underwater robots.

One limitation of the GFM-Planner is that the construction of the MEM relies on pre-built static maps, which makes the framework unsuitable for dynamic environments or unknown scenes. Additionally, the GFM design assumes ideal LiDAR registration conditions, which may not hold in cluttered scenes with occlusions. In our future work, we will address these limitations through online metric updates and robustness analysis under sensor noise.

\appendix

\subsection{Lemmas}
\label{app:lemma}

Let $\mathbf{a}_1, \mathbf{a}_2, \ldots, \mathbf{a}_N$ be nonzero vectors, and let the matrix
\begin{equation}
    \mathbf{A} = \sum_{i=1}^{N} \mathbf{a}_i \mathbf{a}_i^\top.
\end{equation}

\begin{lemma}
    The matrix $\mathbf{A}$ is positive semi-definite.
    \label{lemma:semi}
\end{lemma}

\begin{proof}
    For any nonzero vector $\boldsymbol{\xi}$, we have 
        \begin{equation}
            \boldsymbol{\xi}^\top \mathbf{A} \boldsymbol{\xi} = \sum_{i=1}^{N} \left( \mathbf{a}_i^\top \boldsymbol{\xi} \right)^\top \left( \mathbf{a}_i^\top \boldsymbol{\xi} \right) = \sum_{i=1}^{N} \left( \mathbf{a}_i^\top \boldsymbol{\xi} \right )^2 \ge 0.
        \end{equation}
        Therefore, the matrix $\mathbf{A}$ is positive semi-definite.
\end{proof}

\begin{lemma}
    If the matrix $\mathbf{A}$ is not positive definite, then there exists a nonzero vector $\boldsymbol{\xi}$ such that $\mathbf{a}_i^\top \boldsymbol{\xi} = 0$, for all $i \in \left\lbrace 1, 2, \cdots, N \right\rbrace$.
    \label{lemma:positive}
\end{lemma}

\begin{proof}
    We proceed by contradiction. Suppose that for every nonzero vector $\boldsymbol{\xi}$, there exists at least one index $j$ such that $\mathbf{a}_j^\top \boldsymbol{\xi} \neq 0$. Then, we obtain
        \begin{equation}
            \boldsymbol{\xi}^\top \mathbf{A} \boldsymbol{\xi} = \sum_{i=1}^{N} \left( \mathbf{a}_i^\top \boldsymbol{\xi} \right )^2 \ge \left( \mathbf{a}_j^\top \boldsymbol{\xi} \right)^2 > 0.
        \end{equation}
    This implies that the matrix $\mathbf{A}$ is positive definite, contradicting the given condition that it is not. Therefore, there must exist a nonzero vector $\boldsymbol{\xi}$ satisfying $\mathbf{a}_i^\top \boldsymbol{\xi} = 0$, for all $i \in \left\lbrace 1, 2, \cdots, N \right\rbrace$.
\end{proof}

\subsection{Detailed Derivation of Formula (\ref{eq:result})}
\label{app:details}

The equation of $l_n$ can be expressed as $y = k \left( x - x_0 \right) + y_0$, where $y_0 = k(x_0 - a) + b$. Define $\Delta \mathbf{p} = \left( \Delta x, \Delta y, \Delta \theta \right)^\top$, and let $\mathbf{e}_i$ denote the $i$-th column of the third-order identity matrix. We consider the LiDAR ray equations for different perturbed poses. For the pose $\mathbf{p}_x = \mathbf{p}_0 + \mathbf{e}_1^\top \Delta \mathbf{p}$, the corresponding LiDAR ray equation is $y = k(x-x_0-\Delta x) + y_0$. For the pose $\mathbf{p}_y = \mathbf{p}_0 + \mathbf{e}_2^\top \Delta \mathbf{p}$, the LiDAR ray equation is $y = k(x - x_0) + y_0 + \Delta y$. For the pose $\mathbf{p}_\theta = \mathbf{p}_0 + \mathbf{e}_3^\top \Delta \mathbf{p}$, the LiDAR ray equation is $y = \hat{k}(x - x_0) + y_0$, where
\begin{equation}
    \hat{k} = \tan \left( \arctan k + \Delta \theta \right) = \dfrac{k + \tan \Delta \theta}{1 - k \tan \Delta \theta}.
\end{equation}
The intersection points of these LiDAR rays with the Taylor expansion (\ref{eq:taylor}) are
\begin{subequations}
    \begin{equation}
        \mathbf{L}_n \left( \mathbf{p}_x \right) = \left[ a + \dfrac{Bk \Delta x}{A+Bk}, b - \dfrac{Ak \Delta x}{A+Bk} \right]^\top,
    \end{equation}
    \begin{equation}
        \mathbf{L}_n \left( \mathbf{p}_y \right) = \left[ a - \dfrac{B\Delta y}{A+Bk}, b + \dfrac{A\Delta y}{A+Bk} \right]^\top,
    \end{equation}
    \begin{equation}
        \mathbf{L}_n \left( \mathbf{p}_\theta \right) = \begin{bmatrix}
            \dfrac{a(A+Bk)+Bx_0(\hat k-k)}{A+B \hat k} \\
            b + \dfrac{A(a-x_0)(\hat k-k)}{A+B\hat k}
        \end{bmatrix}.
    \end{equation}
\end{subequations}
Therefore, we have
\begin{subequations}
    \begin{equation}
        \underset{\Delta x \to 0}{\lim} \dfrac{\mathbf{L}_n \left( \mathbf{p}_x \right) - \mathbf{L}_n \left( \mathbf{p}_0 \right)}{\Delta x} = \dfrac{-k}{A+Bk} \mathbf{M},
    \end{equation}
    \begin{equation}
        \underset{\Delta y \to 0}{\lim} \dfrac{\mathbf{L}_n \left( \mathbf{p}_y \right) - \mathbf{L}_n \left( \mathbf{p}_0 \right)}{\Delta y} = \dfrac{1}{A+Bk} \mathbf{M},
    \end{equation}
    \begin{equation}
        \underset{\Delta \theta \to 0}{\lim} \dfrac{\mathbf{L}_n \left( \mathbf{p}_\theta \right) - \mathbf{L}_n \left( \mathbf{p}_0 \right)}{\Delta \theta} = \dfrac{(1+k^2)(a-x_0)}{A+Bk} \mathbf{M},
    \end{equation}
\end{subequations}
where $\mathbf{M} = (-B, A)^\top$. Thus, formula (\ref{eq:result}) follows directly from these derivations.

\bibliographystyle{IEEEtran}
\bibliography{cite}

\end{document}